%% file: manuscript.tex
\documentclass[journal]{IEEEtran}

\usepackage{cite}
\usepackage{amsmath}
\usepackage{amsfonts}
\usepackage{amssymb}
\usepackage{amsthm}
\usepackage{graphicx}
\usepackage{url}
\usepackage{caption}
\usepackage{subcaption}
\usepackage[linesnumbered,ruled,vlined]{algorithm2e}

\usepackage{pifont}

\newtheorem{theorem}{Theorem}
\newtheorem{definition}[theorem]{Definition}
\newtheorem{lemma}{Lemma}

\SetKwInput{KwInput}{Input}                
\SetKwInput{KwOutput}{Output}              
\newlength\myindent
\title{Differentially Private Federated Learning\\for Cancer Prediction}

\author{\IEEEauthorblockN{Constance Beguier, Jean Ogier du Terrail, Iqraa Meah, Mathieu Andreux, Eric W. Tramel\\}
\IEEEauthorblockA{Owkin Inc. \\ New York, USA. \\
Email: firstname.lastname@owkin.com}
}

\begin{document}
\maketitle
\input{abstract}
\input{introduction}
\input{background}
\input{proposed-method}
\input{experiments}
\input{conclusion}
\input{acknowledgements}

\bibliographystyle{IEEEtran}
\bibliography{references}

\end{document}

%% file: abstract.tex
\begin{abstract}
Since 2014, the NIH funded iDASH (integrating Data for Analysis, Anonymization, SHaring)
National Center for Biomedical Computing has hosted
yearly competitions on the topic of private computing for genomic data.
For one track of the 2020 iteration of this competition, participants were challenged 
to produce an approach to federated learning~(FL) training of genomic cancer prediction 
models using differential privacy~(DP), with submissions ranked according to 
held-out test accuracy for a given set of DP budgets.
More precisely, in this track, we are tasked with training a supervised model for the
prediction of breast cancer occurrence from genomic data split between two virtual 
centers while ensuring data privacy with respect to model transfer via DP.
In this article, we present our 3$^{\rm rd}$ place submission to this competition.
During the competition, we encountered two main challenges discussed in this article:
i) ensuring correctness of the privacy budget evaluation
and ii) achieving an acceptable trade-off between prediction performance and privacy budget.
\end{abstract}

%% file: introduction.tex
\section{Introduction}

Machine learning~(ML) is increasingly used in healthcare to solve a large variety of problems,
including disease diagnosis and biomarker discovery~\cite{yu2018artificial}.
In ML, generally, the larger the dataset, the more accurate the trained ML model is.
Thus, one would like to collect the largest possible dataset for a given problem
in a single location.
However, in many contexts and especially in healthcare, data sensitivity
and increasingly stringent regulations complicate the creation of such centralized datasets.

Federated learning (FL)~\cite{FL_SS15,FL_MMRA16}
is a paradigm to train an ML model across several datasets in different locations 
in order to avoid the need to collect training data to a single location.
The main idea is to share model parameters,
aggregating them regularly after local training steps, leading to a shared common model.
Leaving data at its source is a big step into increasing data privacy but is not sufficient.
Indeed, over the past years, many privacy attacks based on the knowledge of
model parameters have been
introduced~\cite{FLattacks_SSM19,FLattacks_SDSO19,FLattacks_ZLH19,FLattacks_ZMB20}.
For instance, a membership inference attack creates an oracle to answer
whether a given sample has been used during the training.

A common framework to secure FL training with some robust privacy guarantees is
differential privacy (DP)~\cite{dwork2014algorithmic}.
The main goal of a DP mechanism is to release some information obtained from a database,
while preserving the privacy of the individuals composing this database.

This paper is organized as follows.
In Sec.~\ref{sec:background}, we review background material on DP
and its applications to ML, while 
in Sec.~\ref{sec:method} we detail the method used in our submission.
Sec.~\ref{sec:experiments} provides the experimental protocol and results obtained.
Finally, in Sec.~\ref{sec:conclusion} we present our conclusions.

The source code of our submission is available online at
\url{https://github.com/owkin/idash-2020}.

%% file: background.tex
\section{Background}\label{sec:background}

\subsection{Differential Privacy}

DP~\cite{dwork2014algorithmic} is a widely used framework to
secure model training and protect training data.
The classical definition of DP, provided below, relies on the notion of so-called 
\emph{adjacent} databases, i.e. databases differing in at most a single element 
(or \emph{sample}, as it is the case in ML datasets).
\begin{definition}[$(\epsilon, \delta)$-Differential Privacy (DP) \cite{dwork2014algorithmic,RDP_Mironov17}]
A randomized mechanism $f: \mathcal{D} \rightarrow \mathcal{R}$ satisfies $(\epsilon, \delta)$-DP,
if for any adjacent $D, D^{\prime} \in \mathcal{D}$ and $S \subset \mathcal{R}$,
\begin{equation}\label{eq:dpdef}
Pr[f(D) \in S] \leq e^{\epsilon} Pr[f(D^{\prime}) \in S] + \delta,
\end{equation}
where $\epsilon$ is the privacy budget and $\delta$ is the failure probability.
\end{definition}
It is worth noting that $(\epsilon, \delta)$-DP definitions satisfy inclusion rules,
as stated in the following lemma, whose proof directly stems from Equation~\eqref{eq:dpdef}.
\begin{lemma}[Order relationships of $(\epsilon, \delta)$-DP]\label{lemma:inclusion_dp}
If $f: \mathcal{D} \rightarrow \mathcal{R}$ satisfies $(\epsilon, \delta)$-DP,
then for any $(\epsilon', \delta')$ such that $\epsilon \leq \epsilon'$ and $\delta \leq \delta'$,
$f$ satisfies $(\epsilon', \delta')$-DP.
\end{lemma}
A simple method to create an $(\epsilon, \delta)$-DP mechanism
out of a function $f$ is to add independent Gaussian noise,
with variance proportional to the sensitivity of the function to be protected.
\begin{definition}[Sensitivity \cite{RDP_Mironov17}]\label{thm:sensitivity}
The sensitivity of a function $f: \mathcal{D} \rightarrow \mathcal{R}$
is defined as
\begin{equation}
\Delta f = \max_{D, D^{\prime} \in \mathcal{D}} || f(D) - f(D^{\prime}) ||_2,
\end{equation}
where the maximum is taken over all adjacent $D, D^{\prime} \in \mathcal{D}$.
\end{definition}
To evaluate the privacy budget spent by a randomized mechanism,
a tight analysis tracking the privacy loss must be performed.
If this analysis is too coarse, the resulting bounds will not be tight enough
and the privacy cost associated with DP operations will be
overestimated.
The R\'enyi Differential Privacy (RDP)\cite{RDP_Mironov17,RDP_SGM_MRZ19},
a relaxation of $(\epsilon, \delta)$-DP,
is a widely used tool to perform this analysis as it leads
to simpler composition results than $(\epsilon, \delta)$-DP.
\begin{definition}[$(\alpha, \epsilon)$-R\'enyi Differential Privacy (RDP) \cite{RDP_Mironov17}]
A randomized mechanism $f: \mathcal{D} \rightarrow \mathcal{R}$ satisfies $(\alpha, \epsilon)$-RDP,
if for any adjacent $D, D^{\prime} \in \mathcal{D}$,
\begin{equation}
D_{\alpha}(f(D) || f(D^{\prime})) \triangleq \frac{1}{\alpha - 1} \log \mathbb{E}_{x \sim f(D^{\prime})} \left( \frac{f(D)}{f(D^{\prime})} \right)^{\alpha} \leq \epsilon
\end{equation}
where $D_{\alpha}$ is the R\'enyi divergence of order $\alpha > 1$.
\end{definition}
It is easy to see that both definitions are closely related.
In most cases, it is difficult to directly evaluate the privacy cost of 
a single complex operation, so instead operations are often evaluated as a
composition of multiple operations whose privacy cost is more easily evaluated.
The composition theorems of DP allow for the direct computation of the 
privacy cost of a sequence of DP operations.

\begin{theorem}[RDP composition \cite{RDP_Mironov17}]\label{thm:rdpcomp}
Let $f: \mathcal{D} \rightarrow \mathcal{R}_1$ be $(\alpha, \epsilon_1)$-RDP
and $g: \mathcal{R}_1 \times \mathcal{D} \rightarrow \mathcal{R}_2$ be $(\alpha, \epsilon_2)$-RDP,
then the mechanism defined as $(X, Y)$ where $X \sim f(D)$ and $Y \sim g(X, D)$
satisfies $(\alpha, \epsilon_1 + \epsilon_2)$-RDP.
\end{theorem}

\begin{theorem}[RDP post-processing \cite{RDP_Mironov17}]\label{thm:postprocessing}
Let $g: \mathcal{R} \rightarrow \mathcal{R}^{\prime}$ be a randomized mapping.
If $f: \mathcal{D} \rightarrow \mathcal{R}$ is $(\alpha, \epsilon)$-RDP,
then $g \circ f$ is also $(\alpha, \epsilon)$-RDP.
\end{theorem}

Finally, the RDP privacy budget can be converted back to
an $(\epsilon, \delta)$-DP privacy budget with the help of the following theorem.

\begin{theorem}[RDP to DP conversion \cite{RDP_Mironov17}]\label{thm:rdp2dp}
If $f$ satisfies $(\alpha, \epsilon)$-RDP,
then $f$ satisfies $(\epsilon + \frac{\log ( 1 / \delta)}{\alpha - 1}, \delta)$-DP
for all $\delta \in (0, 1)$.
\end{theorem}

\subsection{DP-SGD}
DP-SGD \cite{DPSGD_ACGM16} is the most commonly proposed approach for DP 
training of ML models with parameters $\theta$ with stochastic gradient descent (SGD)
via an $(\epsilon, \delta)$-DP mechanism.
At each SGD step of this algorithm, gradient clipping is used to bound the 
sensitivity of the batch-gradient calculation (see Thm. \ref{thm:sensitivity}),
and subsequently Gaussian noise of variance proportional to the selected
sensitivity bound is added to the clipped gradients.
We refer to Algorithm~\ref{algo:DPSGD} for details.

\begin{algorithm}[h!]
\caption{DP-SGD \cite{DPSGD_ACGM16}}
\label{algo:DPSGD}
\DontPrintSemicolon
\KwInput{
    Samples $\{x_1, ..., x_n\}$,
    empirical loss function $\mathcal{L}(\theta) = (1/N) \sum_i \mathcal{L}(\theta, x_i)$.
    Parameters: learning rates $\eta_t$,
    noise scale $\sigma$,
    group size $L$,
    gradient norm bound $C$.
}
\textbf{Initialize} $\theta_0$ randomly\;
\For{$t \in [T]$}{
    Take a random sample $L_t$ with sampling probability $L/N$\;
    \textbf{Compute gradient} \;
    For each $i \in L_t$, compute $g_t(x_i) \leftarrow \nabla_{\theta_t} \mathcal{L}(\theta_t, x_i)$\;
    \textbf{Clip gradients} \;
    $\bar{g}_t(x_i) \leftarrow \frac{g_t(x_i)}{max(1, ||g_t(x_i)||_2 / C)}$\;
    \textbf{Add noise}\;
    $\bar{g}_t \leftarrow \frac{1}{|L|} \left( \sum_i \bar{g}_t(x_i) + \mathcal{N}(0, \sigma^2 C^2 I) \right)$\;
    \textbf{Descent}\;
    $\theta_{t+1} \leftarrow \theta_t - \eta_t \bar{g}_t$
}
\KwOutput{
    $\theta_T$ and compute the overall privacy cost $(\epsilon, \delta)$ using a privacy accounting method.
}
\end{algorithm}

This algorithm has been implemented for PyTorch in Opacus\footnote{\url{https://github.com/pytorch/opacus}}
and for TensorFlow in TensorFlow Privacy\footnote{\url{https://github.com/tensorflow/privacy}}.
In these two implementations, the privacy accounting method is based on the privacy analysis of~\cite{RDP_SGM_MRZ19},
which assumes that samples are drawn according to the \emph{Sampled Gaussian Mechanism} (see Def.~\ref{def:SGM}).
Thus, when using these libraries, we must take care in the creation of batches
in order ensure that this assumption is respected.
In particular, as pointed out in~\cite{ImprovingDLwithDP_NSH20},
using shuffling at each epoch to create batches makes the privacy
budget calculations from~\cite{RDP_SGM_MRZ19} not applicable:
using them leads to incorrect estimates of the privacy cost of DP operations.

\subsection{Privacy Budget Evaluation of DP-SGD from \cite{RDP_SGM_MRZ19}}\label{subsec:RDP_DPSGD}

In this section, we recall the main results of~\cite{RDP_SGM_MRZ19}
allowing one to evaluate the privacy budget of DP-SGD.

\paragraph{Privacy Budget of one DP-SGD Step}
In \cite{RDP_SGM_MRZ19}, privacy budget evaluation is based on the Sampled Gaussian Mechanism (SGM),
whose definition is provided below.
\begin{definition}[Sampled Gaussian Mechanism (SGM) \cite{RDP_SGM_MRZ19}]
\label{def:SGM}
Let $f$ be a function mapping subsets of $\mathcal{S}$ to $\mathbb{R}^d$.
We define the Sampled Gaussian Mechanism (SGM) parameterized with the sampling rate $0<q\leq 1$ and the noise $\sigma>0$ as
\begin{align*}
SG_{q, \sigma} (S) \triangleq & f(\{x: x \in S \text{ is sampled with probability } q\}) \\
& + \mathcal{N}(0, \sigma^2 \mathbb{I}^d)
\end{align*}
\end{definition}

In~\cite{RDP_SGM_MRZ19}, it is first demonstrated how to evaluate the privacy budget of one DP-SGD step
through the lenses of SGM: in DP-SGD, $f$ is the clipped gradient evaluation
in sampled data points $f(\{x_i\}_{i \in B}) = \sum_{i \in B} \bar{g_t}(x_i)$.
If $\bar{g}_t$ is obtained by clipping $g_t$ with a gradient norm bound $C$,
then the sensitivity of $f$ is equal to $C$.

\begin{theorem}[RDP privacy budget of SGM~\cite{RDP_SGM_MRZ19}]\label{thm:rdp_sgm}
Let $SG_{q, \sigma}$ be the Sampled Gaussian Mechanism for some function $f$.
If $f$ has sensitivity $1$, then
$SG_{q, \sigma}$ satisfies $(\alpha, \epsilon)$-RDP whenever
\begin{equation}
\epsilon \leq \frac{1}{\alpha - 1} \log \max(A_{\alpha}(q, \sigma), B_{\alpha}(q, \sigma)),
\end{equation}
where
\begin{equation}
\begin{cases}
A_{\alpha}(q, \sigma) \triangleq \mathbb{E}_{z \sim \mu_0} [(\mu(z) / \mu_0(z))^{\alpha}], \\
B_{\alpha}(q, \sigma) \triangleq \mathbb{E}_{z \sim \mu} [(\mu_0(z) / \mu(z))^{\alpha}],
\end{cases}
\end{equation}
with $\mu_0 \triangleq \mathcal{N}(0, \sigma^2)$, $\mu_1 \triangleq \mathcal{N}(1, \sigma^2)$
and $\mu \triangleq (1-q)\mu_0 + q \mu_1$.

Further, it holds that $\forall (q, \sigma) \in (0,1]$, $\mathbb{R}^{+*}$, $A_{\alpha}(q, \sigma) \geq B_{\alpha}(q, \sigma)$.
Thus, $SG_{q, \sigma}$ satisfies $(\alpha, \frac{1}{\alpha - 1} \log (A_{\alpha}(q, \sigma)))$-RDP.
\end{theorem}

Finally, \cite{RDP_SGM_MRZ19} describes a procedure to compute $A_{\alpha}(q, \sigma)$ depending on $\alpha$.

\textsc{Case I: Integer $\alpha$}
\begin{equation}
A_{\alpha} = \sum_{k=0}^{\alpha} \binom{\alpha}{k} (1-q)^{\alpha-k} q^k \exp \left(\frac{k^2 - k}{2 \sigma^2}\right)
\end{equation}

\textsc{Case II: Fractional $\alpha$}
\begin{multline}
A_{\alpha} = \sum_{k=0}^{\infty} \binom{\alpha}{k} \frac{e^{\frac{k^2 - k}{2 \sigma^2}}}{2}
\left[ (1 - q)^{\alpha - k} q^k \text{ erfc} \left( \frac{k - z_1}{\sqrt{2} \sigma} \right) \right.\\
\left. + (1-q)^k q^{\alpha - k} \text{ erfc} \left( \frac{z_1 - k }{\sqrt{2} \sigma} \right) \right],
\end{multline}
where $z_1 = \frac{1}{2} + \sigma^2 \ln (q^{-1} - 1)$.

In the case of DP-SGD, the gradient function on which the SGM is applied
has a sensitivity bounded by $C$ and not $1$.
However, one can always renormalize this function by $C$ to achieve sensitivity $1$, apply the SGM
and post-process the result by re-multiplying it by $C$, yielding a result equivalent
to DP-SGD.
Due to Theorem~\ref{thm:postprocessing}, the resulting DP-SGD mechanism has the same privacy budget as
the SGM applied on the normalized gradient function.

\paragraph{Privacy Budget of the Overall DP-SGD Training}
By using the RDP composition theorem, applying DP-SGD for~$T$ steps with batches sampled
according to the SGM yields a mechanism satsifying $(\alpha, \frac{T}{\alpha - 1} \log (A_{\alpha}(q, \sigma)))$-RDP. 
This bound can then be converted to an $(\epsilon, \delta)$-DP budget using Theorem~\ref{thm:rdp2dp}.

In practice, in Opacus and TensorFlow Privacy, the RDP privacy budget
is evaluated for multiple values $(\alpha_k)_{k \in [K]}$, yielding
a sequence of RDP privacy budgets $\{(\alpha_k, \epsilon_k)\}_{k \in [K]}$.
Then, given an input failure probability $\delta$, each RDP budget is
mapped to an $(\epsilon, \delta)$-DP privacy budget,
yielding a sequence~$\{(\epsilon_k, \delta)\}_{k \in [K]}$.
Finally, the best privacy budget is provided:~$(\min_{k \in [K]} \epsilon_k, \delta)$.

\subsection{Correctly Evaluating the Privacy Budget in ML}\label{subsec:correctly_evaluating}

The main challenge when building and implementing a randomized DP mechanism
is to be sure that its privacy budget is correctly evaluated.
Almost all DP implementations rely on DP-SGD with the analysis of~\cite{RDP_SGM_MRZ19}
to evaluate the privacy budget.
However, common data science procedures used to build the batches, 
fill missing data, perform feature selection, or manage imbalanced training datasets are 
often incompatible with this privacy budget evaluation
or will drastically increase the overall privacy budget.
In this section, we provide details on each of these crucial steps in light of privacy requirements.

\paragraph{Batch construction}
The analysis of~\cite{RDP_SGM_MRZ19} recalled in Sec.~\ref{subsec:RDP_DPSGD} is based on Sampled Gaussian Mechanism.
As a consequence, this analysis assumes that to build one batch, each sample in the dataset is selected with probability $q$.
In addition, each batch is built independently from the other batches.
Using another batch construction will break this assumption
and the privacy budget will then be incorrectly evaluated.
In particular, one cannot sample constant-sized batches through random dataset shuffling at each epoch
and use the privacy accounting method from~\cite{RDP_SGM_MRZ19}.
In our approach, we keep the batch construction which is used in~\cite{RDP_SGM_MRZ19}.

\paragraph{Missing data imputation}
If the dataset contains some missing values, a common method to fill them consists
in evaluating the mean or median of each feature in the training dataset and filling
missing values in training dataset and test dataset with the corresponding mean or median.
However, without modification, this algorithm does not satisfy $(\epsilon, \delta)$-DP.
Indeed, revealing the mean and median of each feature in the training dataset without modification
would lead to potential privacy leaks.
A direct solution could be to add noise to these statistics
in order to make the resulting mechanism $(\epsilon, \delta)$-DP compliant.
This operation would then need to be included in the estimation 
of the overall privacy budget.

Instead, we choose to favor privacy over precision, and fill the missing values
with zeros, therefore avoiding impacting the privacy budget.
We emphasize that if one uses part of the data to infer the missing values, then this 
has to be reflected in the privacy accounting.

\paragraph{Feature selection}
Most genomic datasets contain many more features than samples.
Feature selection is often applied on this kind of dataset to reduce the number of
predictive features in an attempt to avoid overfitting.
In addition, this feature selection reduces the number of required model parameters, 
often decreasing the overall time-to-train.
Feature selection methods such as principal component analysis (PCA) are applied on the training set,
and then only the selected components are retained for the training and test sets.
However, revealing the selected features, or the selected embedding space, in order to apply 
them on test set does not satisfy $(\epsilon, \delta)$-DP.
In order to securely use these feature selection methods, one would
need to allocate a part of the overall privacy budget to them
and to transform the feature selection method by adding noise in order to obtain
an $(\epsilon, \delta)$-DP mechanism.

In this competition, we would like to predict breast cancer occurrence based on genomic data.
Many public works~\cite{citbcmst,rotterdam} have already selected the best genes to use for this task.
Our feature selection is based on these selected public genes.
Thus, we do not need to allocate a part of the privacy budget to the feature selection.

\paragraph{Imbalanced training dataset}
A common method to manage imbalanced training dataset is oversampling.
This method consists to artificially duplicate in the training set some samples
in order to have the same number of samples in each category.
Unfortunately, the evaluation of the privacy budget with a training set with duplicate samples is not easy.
In our approach, we do not modify the training set
in order to ensure that the evaluation of the privacy budget is correct.

%% file: proposed-method.tex
\section{Proposed Method}\label{sec:method}

\subsection{Our Algorithm: DP-SGD with Cyclic FL}

Our proposed method is based on DP-SGD with a cyclic FL strategy.
The first client randomly initializes the model parameters.
For each federated round, the first client trains the model on its local dataset for $E$ DP-SGD steps,
then it sends the updated model parameters to the second client.
The second client trains the model on its local dataset during $E$ DP-SGD steps,
and sends back the updated model parameters to the first client.
$N$ such federated rounds are performed overall.
Algorithm~\ref{algo:walk_DPSGD} details the proposed method.

\begin{algorithm}[h!]
\caption{DP-SGD with cyclic FL}
\label{algo:walk_DPSGD}
\DontPrintSemicolon
\KwInput{
    \textbf{Common inputs:}
    loss function $\mathcal{L}: (\theta, x) \mapsto \mathcal{L}(\theta, x)$,
    sample rate $q$,
    learning rate $\eta$,
    noise scale $\sigma$,
    gradient norm bound $C$,
    number of FL rounds $N$,
    number of local batch updates per FL round $E$,
    list of selected genes (from a public work) $G$.\;
    \textbf{Client $1$ inputs:} samples $D^1=\{x_1^1, ..., x_{s_1}^1\}$.\;
    \textbf{Client $2$ inputs:} samples $D^2=\{x_1^2, ..., x_{s_2}^2\}$.\;
}
\textbf{Prepare dataset locally on each client $c$}\;
Feature selection: $D^c_G = \{x_i^c[G]\}_{i \in [s_c]}$\;
Fill missing data independently on each client dataset.\;
\textbf{Initialization}\;
Client 1 randomly initializes the model parameters $\theta$\;
\For{$t \in [N]$}{
    \For{$c \in [2]$}{
        \For{$j \in [E]$}{
            \textbf{Create batch}\;
            $B_{t,j}^c = \{x: x\in D^c_G \text{ sampled w/ proba. } q\}$\;
            \textbf{Compute gradients}\;
            $\forall x \in B_{t,j}^c, g_{t, j}^c(x) \leftarrow \nabla_{\theta} \mathcal{L}(\theta, x)$\;
            \textbf{Clip gradients} \;
            $\forall x \in B_{t,j}^c, \bar{g}_{t,j}^c(x) \leftarrow \frac{g_{t,j}^c(x)}{\max(1, ||g_{t,j}^c(x)||_2 / C)}$\;
            \textbf{Add noise}\;
            $n_{t, j}^c \sim \mathcal{N}(0, \sigma^2 C^2 I)$\;
            $\bar{g}_{t,j}^c \leftarrow \frac{1}{|B_{t,j}^c|}\left(\sum_{x \in B_{t,j}^c} \bar{g}_{t,j}^c(x) + n_{t, j}^c \right)$\;
            \textbf{Descent}\;
            $\theta \leftarrow \theta - \eta \bar{g}_{t,j}^c$\;
        }
        Client $c$ sends $\theta$ to the other client\;
    }
}
\KwOutput{
    $\theta$ and compute the overall privacy cost $(\epsilon, \delta)$
    using the method described in Section~\ref{subsec:privacy_evaluation}.
}
\end{algorithm}

\subsection{Privacy Budget Evaluation}
\label{subsec:privacy_evaluation}

Before presenting our privacy budget evaluation, we need to define the notion of DP in FL.
In FL, each participant is only interested in the privacy of its own dataset.
Thus, we evaluate the privacy budget of each participant separately from the point of view of a
single participant only. 
In particular, all steps performed by the other client can be considered as post-processing.
Since the computations performed at each client are identical for 
this challenge, estimating the privacy budget is symmetric with respect to the two participants,
thus understanding the privacy budget at a single arbitrary participant is all we need to know to
understand the privacy budget at all participants. 

\begin{theorem}[RDP Privacy Budget of Algorithm~\ref{algo:walk_DPSGD}]
\label{th:our_privacy_budget}
Algorithm~\ref{algo:walk_DPSGD} satisfies $(\alpha, N \cdot E \cdot \frac{1}{\alpha-1} \log (A_{\alpha}(q, \sigma)))$-RDP.
\end{theorem}

\begin{proof}
We will only prove the privacy budget from the point of view of the client $1$.
The analysis of the privacy budget from the point of view of the client $2$ is similar.

Our privacy budget evaluation is only based on the analysis from \cite{RDP_SGM_MRZ19}
and Theorem~\ref{thm:postprocessing} related to RDP post-processing.
By applying Theorem~\ref{thm:rdp_sgm} and Theorem~\ref{thm:rdpcomp},
the RDP privacy budget after the $E$ first DP-SGD steps
is equal to $(\alpha, E \cdot \frac{1}{\alpha-1} \log (A_{\alpha}(q, \sigma)))$.
Then the~$E$ local steps performed by client~$2$ are considered as post-processing from client~$1$'s point of view.
Thus, the RDP privacy budget of one FL round for client~$1$
is equal to $(\alpha, E \cdot \frac{1}{\alpha-1} \log (A_{\alpha}(q, \sigma)))$.

Finally, by applying the RDP composition theorem~\ref{thm:rdpcomp}, the RDP privacy budget of the overall training
is equal to $(\alpha, N \cdot E \cdot \frac{1}{\alpha-1} \log (A_{\alpha}(q, \sigma)))$.
\end{proof}

As in Opacus and TensorFlow Privacy, we evaluate the~$(\epsilon, \delta)$-DP privacy budget by
first, evaluating the~$(\alpha, \epsilon)$-RDP privacy budget
for various $\alpha$'s using Theorem~\ref{th:our_privacy_budget},
then converting each~$(\alpha, \epsilon)$-RDP privacy budget into a corresponding~$(\epsilon, \delta)$-DP privacy budget
based on Theorem~\ref{thm:rdp2dp} and finally by keeping the best~$(\epsilon, \delta)$-DP privacy budget.

Our experiments are based on Opacus library.
Unfortunately, this library is incompatible with the batch construction based on the Sampled Gaussian Mechanism.
Thus, we patched it in order to obtain a correct privacy budget evaluation. 

%% file: experiments.tex
\section{Experiments}\label{sec:experiments}

\subsection{Dataset and Model Architecture}

\paragraph{Dataset Composition}
In the competition, a single dataset containing
61 normal samples and 529 tumorous samples is provided.
Each sample has $17, 814$ genes.
We randomly split this dataset using label stratification into
one training set ($90\%$) and one test set ($10\%$).

\paragraph{Dataset Pre-processing}
As explained in Section~\ref{subsec:correctly_evaluating},
our gene selection is not based on the training dataset
in order to keep the overall privacy budget for the training.
For our gene selection, we test two public genes signatures:
\begin{itemize}
\item the \texttt{Rotterdam} signature~\cite{rotterdam} which contains 89 genes,
out of which only 69 are present in the competition dataset;
\item the \texttt{citbcmst} signature~\cite{citbcmst} which contains 257 genes,
out of which 240 are present in the competition dataset.
\end{itemize}

As explained Section~\ref{subsec:correctly_evaluating}, in order to fill the missing data,
we impute missing entries with zeros.

\paragraph{Model Architecture}
We observed that the classification task is prone to overfitting due to the ease of the breast cancer detection task.
Thus, we only train simple models: a linear logistic regression model and a shallow multi-layer perceptron (MLP).
For all tested privacy budgets, we obtain a better performance with the Logistic Regression model
than with the shallow MLP.

\paragraph{Metric}
To evaluate this binary classification task in spite of the dataset imbalance
we use the challenge's unweighted accuracy metric, defined as
\begin{equation}\label{eq:accuracy}
\mathrm{Accuracy} = \frac{\mathrm{TP} + \mathrm{TN}}{\mathrm{TP} + \mathrm{TN} + \mathrm{FP} + \mathrm{FN}},
\end{equation}
where $\mathrm{TP}$ is the number of True Positives, $\mathrm{TN}$ the number of True Negatives,
$\mathrm{FP}$ the number of False Positives and $\mathrm{FN}$ the number 
of False Negatives in the test dataset.
In other words, accuracy~\eqref{eq:accuracy} is the ratio between the number of samples
correctly classified and the total number of samples.
This accuracy is the metric used by the competition organizers.

\begin{figure*}[t!]
    \centering
         \begin{subfigure}[b]{0.32\textwidth}
             \centering
             \includegraphics[width=\textwidth, trim={0.5cm, 0.3cm, 1.6cm, 0.5cm}, clip]{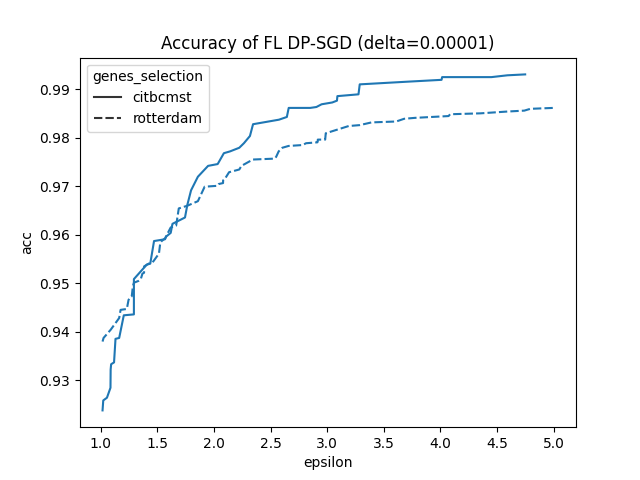}
             \caption{$\delta=10^{-5}$}
         \end{subfigure}
         \hfill
         \begin{subfigure}[b]{0.32\textwidth}
             \centering
             \includegraphics[width=\textwidth, trim={0.5cm, 0.3cm, 1.6cm, 0.5cm}, clip]{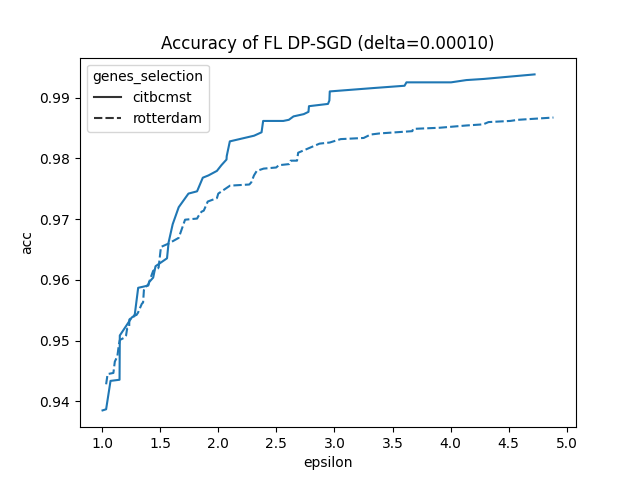}
             \caption{$\delta=10^{-4}$}
         \end{subfigure}
         \hfill
         \begin{subfigure}[b]{0.32\textwidth}
             \centering
             \includegraphics[width=\textwidth, trim={0.5cm, 0.3cm, 1.6cm, 0.5cm}, clip]{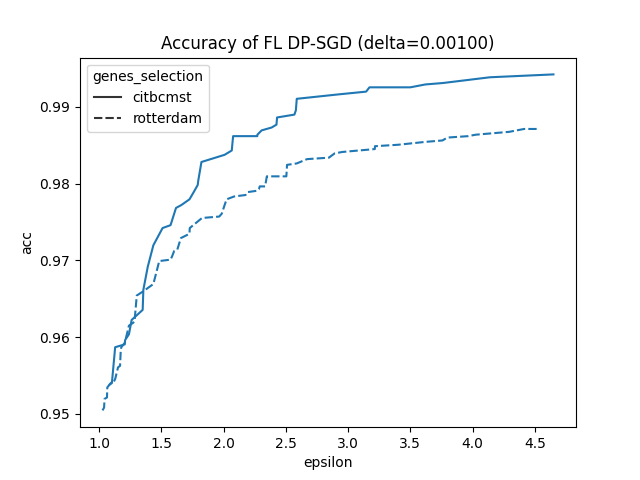}
             \caption{$\delta=10^{-3}$}
         \end{subfigure}
            \caption{Average validation accuracy over 50 independent trials for different privacy budgets.}
            \label{fig:fine_tuning_results}
    \end{figure*}

\subsection{Parameters Selection and Results}

\paragraph{Set of Parameters}
There are a number of parameters to fine-tune for the proposed procedure, e.g.
\begin{itemize}
\item the gene signature (\texttt{Rotterdam} or \texttt{citbcmst}),
\item the sample rate $q$,
\item the learning rate $\eta$,
\item the noise scale $\sigma$,
\item the gradient norm bound $C$,
\item the number of federated rounds $N$, and
\item the number of local batch updates per federated round, $E$.
\end{itemize}
We performed a grid search to fine-tune these parameters.

\paragraph{Training Process}
For each experiment, we randomly split the training set using label stratification into
the client~$1$'s training set ($40\%$), the client $2$'s training set ($40\%$) and a holdout validation set ($20\%$).
Then each client locally prepares their training set by applying gene selection and filling missing data.
Then the two clients perform an FL DP-SGD training following Algorithm~\ref{algo:walk_DPSGD}.
Finally, we evaluate the accuracy of the trained model on the validation set.

For each set of parameters, the above process is performed~50 times with different random seeds
to split the training set. Each seed also controls model parameter initialization, batch sampling,
and noise realization added to gradients.
Then, the mean accuracy over the~50 experiments is recorded.
For each set of parameters, the corresponding $(\epsilon, \delta)$-DP privacy budget is evaluated
from the parameters~$(q, \sigma, N, E)$.

Note that the goal of taking the mean accuracy over 50 runs is to evaluate the expected performance more accurately
as training is sensitive to all sources of randomness (splits, initialization, batch sampling and DP noise).

\paragraph{Results}
We evaluate the performance of the proposed algorithm under a varying
$(\epsilon, \delta)$-DP privacy budget, reported in Figure~\ref{fig:fine_tuning_results}.

Results are obtained as follows.
We first define a grid of privacy budget parameters~$\lbrace (\epsilon_i, \delta_i) \rbrace_i$.
Next, for each value~$(\epsilon_i, \delta_i)$, we compute the best mean validation accuracy obtained
by our algorithm when varying its hyperparameters.
Note that we restrict this search to hyperparameters
yielding a privacy budget~$(\epsilon, \delta)$ satisfying~$\epsilon \leq \epsilon_i$
and $\delta \leq \delta_i$, following Lemma~\ref{lemma:inclusion_dp}.
This allows us to plot the curves in Figure~\ref{fig:fine_tuning_results}.
We additionally save each of these optimal experiments in a comma-separated values (CSV) file containing,
for each experiment, its set of hyperparameters, the actual privacy budget~$(\epsilon, \delta)$
it yields and the mean accuracy over the validation set.
This CSV file is used in Sec.~\ref{subsec:privacy2params}

In Figure~\ref{fig:fine_tuning_results}, we see that for
a low privacy budget ($\epsilon=1$ and~$\delta=10^{-5}$), we obtain an accuracy of $93.5\%$.
When the privacy budget increases, the accuracy quickly reaches its maximum, around $99.5\%$.
Further, for all values of~$\delta$, we observe that for low~$\epsilon$,
training with the \texttt{Rotterdam} signature allows to obtain higher accuracy
than using the \texttt{citbcmst} signature, while the contrary takes place for larger values of~$\epsilon$.
We suspect that the small number of features in the \texttt{Rotterdam} might be helpful in reducing the variance in the low~$\epsilon$ regime,
while for larger~$\epsilon$, the small amount of noise added makes it possible to exploit all the information stored
in the \texttt{citbcmst} signature.

\subsection{From privacy budget to set of parameters}\label{subsec:privacy2params}
In the competition, our training algorithm should take as input target privacy parameters~$(\epsilon_t, \delta_t)$
and output a mean accuracy, not the other way around.
We now explain how we use the CSV file introduced in the previous section to tackle this aspect of the competition.

We first sort the CSV file based on values of~$(\delta, \epsilon)$ in lexicographic order,
i.e. first with respect to~$\delta$,
and then with respect to~$\epsilon$.
Then, given target privacy parameters~$(\delta_t, \epsilon_t)$, we select the experiment contained in the CSV file
with the largest value of~$(\delta, \epsilon)$ according to lexicographic order,
under the constraint that~$(\delta, \epsilon) \leq (\delta_t, \epsilon_t)$.
We provide the set of parameters of the selected experiment to the algorithm introduced in Section~\ref{sec:method},
ultimately yielding an accuracy.

%% file: conclusion.tex
\section{Conclusion}\label{sec:conclusion}

Throughout this competition, we faced two main challenges:
\begin{itemize}
\item ensuring correctness of the privacy budget evaluation;
\item achieving a good trade-off between accuracy and the privacy budget.
\end{itemize}

Regarding the correctness of the privacy budget evaluation,
our algorithm is based on the well known and proved algorithm DP-SGD~\cite{DPSGD_ACGM16},
while keeping the batch construction of~\cite{RDP_SGM_MRZ19}
in order not to invalidate the privacy budget analysis of this last paper.
Throughout our experiments we make sure that the only operation impacting our
privacy budget is the actual training by ensuring that data preprocessing
doesn't leak private information.
Last, but not least, the proposed FL algorithm extends DP-SGD in the FL setting,
keeping a good balance between accuracy and privacy.

Our experiments show that with our algorithm it is possible
to achieve a good trade-off between accuracy and privacy budget.
We obtain $93.5\%$ of accuracy for a stringent privacy budget ($\epsilon=1$ and $\delta=10^{-5}$)
and the accuracy quickly converges to $99.5\%$ when $\epsilon$ increases.

%% file: acknowledgements.tex
\section*{Acknowledgements}
We thank Aurelie Kamoun and Alberto Romagnoni who helped selecting the genes signatures used in our submission.